\documentclass[12pt,final]{l4dc2022}

\usepackage{amsfonts}       %
\usepackage{graphicx}
\usepackage{epstopdf}
\usepackage{mathtools}
\usepackage{dsfont}
\usepackage{tikz}
\usepackage{siunitx}
\usepackage{hyperref}
\usepackage[font=footnotesize]{caption}
\usepackage[noabbrev,capitalize]{cleveref}
\usepackage{balance}    %
\usepackage{xcolor}
\usepackage{tcolorbox}
\usepackage{algorithm}
\usepackage{algpseudocode}
\usepackage{mathtools}
\usepackage{dsfont}
\usepackage{mathrsfs}
\usepackage[noadjust]{cite} %
\usepackage{stmaryrd}
\usepackage{placeins}
\usepackage{upgreek}
\usepackage{wrapfig}
\usepackage{bm}
\usepackage{noel}

\title[Neural Gaits]{Neural Gaits: Learning Bipedal Locomotion via \\ Control Barrier Functions and Zero Dynamics Policies}
\usepackage{times}

\author{%
 \Name{Ivan Dario {Jimenez Rodriguez}$^{*\dagger}$} \Email{ivan.jiemenez@caltech.edu} \\
 \Name{Noel Csomay-Shanklin$^{*\dagger}$} \Email{noelcs@caltech.edu} \\
  \Name{Yisong Yue$^{\dagger\ddag}$} \Email{yyue@caltech.edu} \\
  \Name{Aaron D. Ames$^{\dagger}$} \Email{ames@caltech.edu}\\
 \addr $^\dagger$Caltech, Pasadena, CA, USA \quad\quad $^\ddag$Argo AI%
}

\begin{document}

\maketitle

\vspace{-6mm}
\begin{abstract}%
This work presents Neural Gaits, a method for learning dynamic walking gaits through the enforcement of set invariance that can be refined episodically using experimental data from the robot.
We frame walking as a set invariance problem enforceable via control barrier functions (CBFs) defined on the reduced-order dynamics quantifying the underactuated component of the robot: the zero dynamics.  
Our approach contains two learning modules: one for learning a policy that satisfies the CBF condition, and another for learning a residual dynamics model to refine imperfections of the nominal model.  
Importantly, learning only over the zero dynamics significantly reduces the dimensionality of the learning problem while using CBFs allows us to still make guarantees for the full-order system. 
The method is demonstrated experimentally on an underactuated bipedal robot, where we are able to show agile and dynamic locomotion, even with partially unknown dynamics.
\end{abstract}


\blfootnote{$^*$ These authors contributed equally to this work.}
\vspace{-4mm}
\section{Introduction}

Realizing bipedal locomotion on legged robots is difficult due to the compounded complexity of nonlinear underactuated dynamics coupled with the hybrid nature of walking. 
Underactuation makes the application of classic nonlinear control approaches challenging, necessitating the use of offline optimization to generate periodic walking gaits. Due to the combinatorics of contact conditions resulting from the hybrid dynamics, feasibility of this optimization problem requires either fixing the contact times and positions (which can be vulnerable to perturbations) or expensive planning through the set of possible contact points.
Pushing this offline optimization problem online allows for reactive controllers but requires the use of reduced-order models that limit formal guarantees. 
Despite impressive examples of implementations that deal with bipedal walking in practice, general bipedal locomotion with formal performance guarantees remains an open problem.

\textbf{Prior Work in Control.} In the control literature, bipedal locomotion follows two general branches: walking with guarantees of stability and predictive control approaches. 
Walking with guarantees usually relies on solving optimization programs offline to generate stable (periodic) gaits \citep{Hereid2017FROST}.
Above all, this approach relies on constraining walking to be a periodic orbit with assumed exponential stability on the underactuated coordinates of the robot.
This underlying assumption can be problematic in safety critical settings when the gait must satisfy hard constraints such as staying on predetermined stepping stones \citep{csomay-shanklin_episodic_nodate,NGUYEN2015147} and also precludes different walking modes such as period-two walking and, more generally, aperiodic locomotion \citep{xiong2019orbit,ames_first_2017}.
This is particularly important given that disturbance rejection can require aperiodic behaviors \citep{4307016}.
Predictive control approaches on the other hand are able to avoid the aforementioned limitations by planning trajectories and/or policies online, and have shown great promise for quadrupedal robots \citep{di2018dynamic, grandia2019feedback}. Their application to bipedal robots is comparatively sparse however, and has predominantly required static stability \citep{tedrake2015closed,scianca2020mpc} or simplified models to mitigate the computational complexity \citep{kuindersma2016optimization, apgar2018fast, xiong20213d}. This leads to challenges when seeking formal guarantees for dynamic bipedal locomotion in the presence of model mismatch between the planning and low-level control layers.

\textbf{Prior Work in Learning.}
Prior work in machine learning has produced impressive results towards realizing legged locomotion using reinforcement learning \citep{lee_learning_2020,DBLP:journals/corr/abs-2011-01387,DBLP:journals/corr/abs-2103-15309,heess2017emergence}.
These methods use relatively simple reward functions along with sophisticated simulations to generate large amounts of data to train a policy capable of traversing a variety of terrains.
Still, these algorithms can be fragile when facing environments outside of the training dataset and are data inefficient due to not exploiting the full dynamics structure.
These challenges make it difficult to reliably apply these methods on complex hardware systems.
Other works have attempted to use reinforcement learning to train a parameterization of a Control Barrier Function (CBF) Control Lyapunov Function (CLF) Quadratic Program Controller (CBF CLF QP) \citep{choireinforcement,csomay-shanklin_episodic_nodate}. Differing from these works, this paper does not learn a projection of the modeling error onto the CLF and CBF constraints; instead we learn the projection of our modeling error on the zero dynamics. Furthermore, we do not specify a desired trajectory but rather provide a set of barriers that imply walking as emergent behavior.

\begin{figure}
    \centering
    \includegraphics[width=\textwidth]{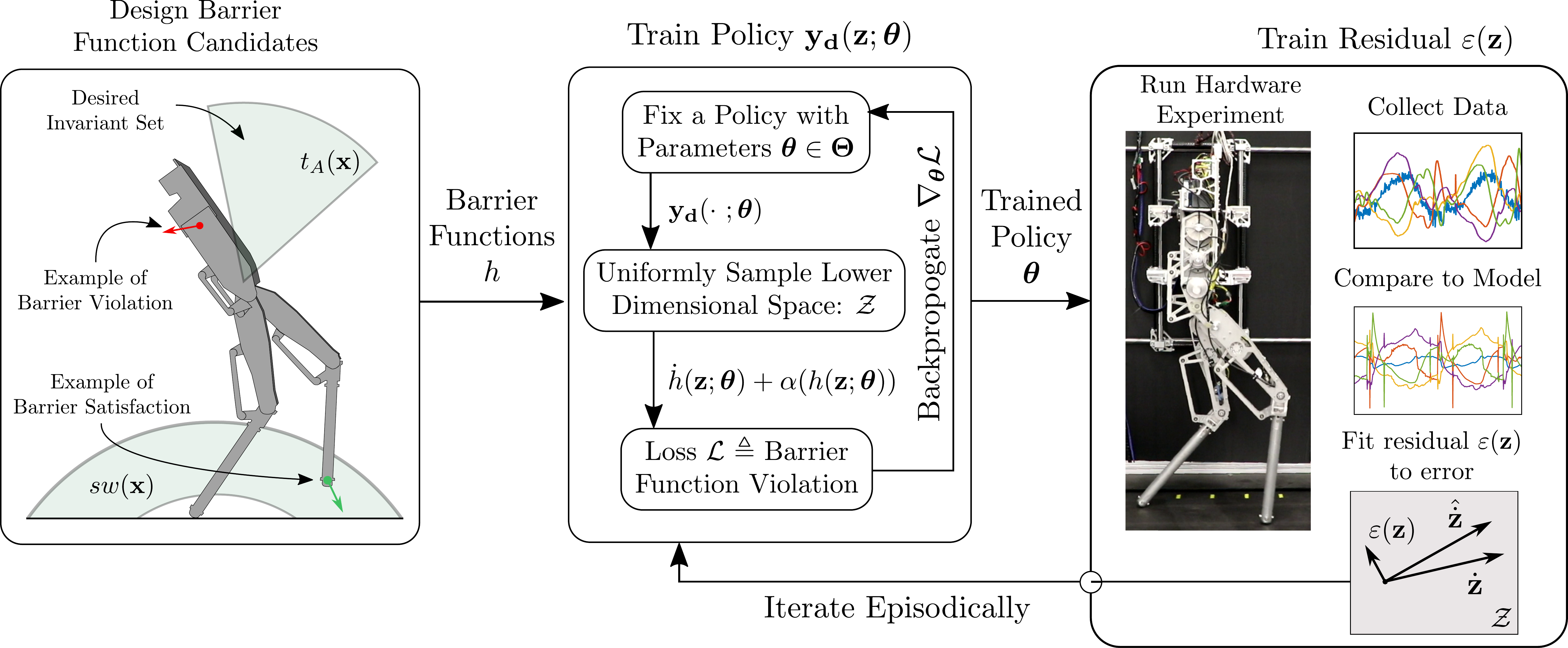}
    \vspace{-20pt}
    \caption{A depiction of the Neural Gaits framework. \textbf{Left:} Designing barrier function candidates that we use to formally describe walking. \textbf{Middle:} Training a policy capable of satisfying all the barrier condition in the zero dynamics state space where the constraint is enforced.
    \textbf{Right:} Collecting hardware data to train a residual zero dynamics model. We then refine the policy episodically using the augmented model.}
    \label{fig:overview}
\end{figure}

\textbf{Our Contributions.}
In this work we are instead interested in automatically discovering good walking policies by integrating learning with control-theoretic formulations of stable walking.
Rather than optimizing a simple reward/cost function, our policies learn to satisfy algebraic forward set invariance conditions, certified by control barrier functions, that characterize good walking behavior (as illustrated in the first block of \cref{fig:overview}).
This integrated approach offers three potential benefits.
First, we are able to certifiably handle impact uncertainty by training policies that satisfy conditions for good impact behavior on a region surrounding the nominal impact guard.
Second, our sampling-based approach is focused on a lower-dimensional state space which can be more data efficient and, under the assumption that walking is a form of set invariance, can produce policies with performance certificates.
The latter point contrasts with offline or predictive trajectory generation perspectives where, even if you enforce control theoretic conditions over the predicted trajectory, you cannot certify performance over a set surrounding the trajectory, thereby breaking set invariance guarantees even under small perturbations. Furthermore, leveraging zero dynamics significantly reduces the dimensionality of the learning problem while remaining compatible with the use of control barrier functions, thus enabling guarantees for the full-order system (see \cref{sec:learn_eps}). 
Third, rather than relying on the artificial constraint of periodic orbits, we are able to characterize walking solely as set invariance while retaining stability guarantees. 
The combined benefit is a reliable and efficient approach for designing gaits that can be deployed on hardware platforms.

Our proposed approach, called \emph{Neural Gaits}, is composed of two learning modules.
The first module trains a policy (which generates walking gaits) to minimize the violations of forward set invariance, implemented using control-theoretic barrier conditions as shown in the second block of \cref{fig:overview}.
In doing so, we can guarantee forward set invariance, which implies (under suitable assumptions) indefinite stable walking.
The second module trains a residual dynamics model to refine imperfections of the current dynamics model based on hardware experiments as represented by the arrow between block three and two in \cref{fig:overview}.
Both modules are then iterated on episodically with hardware experiments in the loop.
We also build upon recent work in training ODE-based systems (of which locomotive walking is an instance of), such as LyaNet \citep{lyanet} and Neural ODEs \citep{chen_neural_2019}, in order to develop an effective training approach.

We empirically demonstrate our approach on the AMBER-3M hardware platform \citep{ambrose_toward_2017} with partially unknown dynamics. We show that the resulting policy is capable of making the robot walk under significant model mismatch and adapts to improve barrier satisfaction across episodes. To the best of our knowledge, this is the first successful demonstration of integrated learning and control for bipedal locomotion with stability guarantees.

\section{Preliminaries}
\label{sec:prelim}
We provide a brief introduction of zero dynamics, hybrid dynamical systems, and control barrier functions, which are necessary fundamentals to understand the proposed formulation in \cref{sec:proposed_method}.
\subsection{Output and Zero Dynamics}
\label{sec:prelim:output_and_zdyn}
Consider the general nonlinear ordinary differential equation:
\begin{align}
    \dot{\gs} &= \gf(\gs, \gu),\quad\quad
    \gs(0) = \gs_0, \label{eq:general_dyn}
\end{align}
with states $\gs \in \gX \subseteq \R^n$, inputs $\gu \in \gU \subseteq \R^m$, and dynamics $\bm f: \gX \times \gU \to \R^n$ with $\bm f$ locally Lipschitz in both arguments. For mechanical systems, we specialize to the control-affine case:
\begin{align}
    \label{eqn:eom}
    \dot{\bm x}  &= \bm f(\bm x) + \bm g(\bm x) \bm u,
\end{align}
where $\bm f: \pX \to \R^n$ and $\bm g: \pX \to \mathbb{R}^{n\times m}$ are assumed to be locally Lipschitz. Denoting a parameter in a parameter space $\bm \theta\in\bm\Theta$, we can define a collection of $k$ outputs $\bm y: \pX \times \bm{\Theta} \to \R^k$ parameterized by $\bm \theta$ that we would like to converge to zero as:
\begin{align}
\label{eq:output_coords}
        \bm y(\bm x; \bm \theta) &= \bm {y_a}(\bm x) - \bm {y_d}(\bm x; \bm{\theta}),
\end{align}
where $\bm{y_a}:\pX\to\R^k$ are the measured outputs, and $\bm{y_d}:\pX \times \bm{\Theta} \to\R^k$ are the desired outputs. For locomotion, the outputs are typically taken either as joint angles (as done in this work), or as center of mass and foot positions. For the policy $\bm y_d$ learned in this work and shown in the center block of \cref{fig:overview}, $\bm{\theta}$  corresponds to neural network parameters.
Although all the concepts may be extended to systems with valid decomposition into output and zero dynamics coordinates (which includes all mechanical systems), for simplicity the remainder of the exposition will be restricted to the setting used in this work, namely with $k=4$ and $\bm y_a$ taken to be the actuated joint angles of the robot. For a complete description of output coordinates and zero dynamics, we  refer to \citep{isidori_elementary_1995}.

Given these outputs $\bm y$, we can separate the actuated and the unactuated coordinates for the robot, which are shown in Figure~\ref{fig:hzd_decomposition}a. As these outputs are \textit{vector relative degree} 2, we can define error coordinates $\bm \eta_i:\pX \to \mathcal{N}_i \subseteq \R^2$ for $i=1,\ldots 4$ as $\bm \eta_i = \begin{bmatrix} y_i^\top, & \dot{y}_i^\top \end{bmatrix}^\top$, as well as the collection of errors $\bm \eta = \begin{bmatrix} \bm\eta_1^\top, & \dots, & \bm\eta_4^\top \end{bmatrix}^\top$. Then, there exist 2 linearly independent functions $z_i:\pX\to\mathcal{Z}_i\subseteq\R$ for $i=1,2$ such that $\nabla_{\mathbf{x}}z_i(\bm x)g(\bm x) \equiv 0$, and  $\nabla_{\bm x} z_i(\bm x)$ is linearly independent from $\nabla_{\bm x} \eta_{i,j}(\bm x)$ for $i=1,\ldots,4$ and $j=1,2$.%
\begin{figure}
    \centering
    \includegraphics[width=\textwidth]{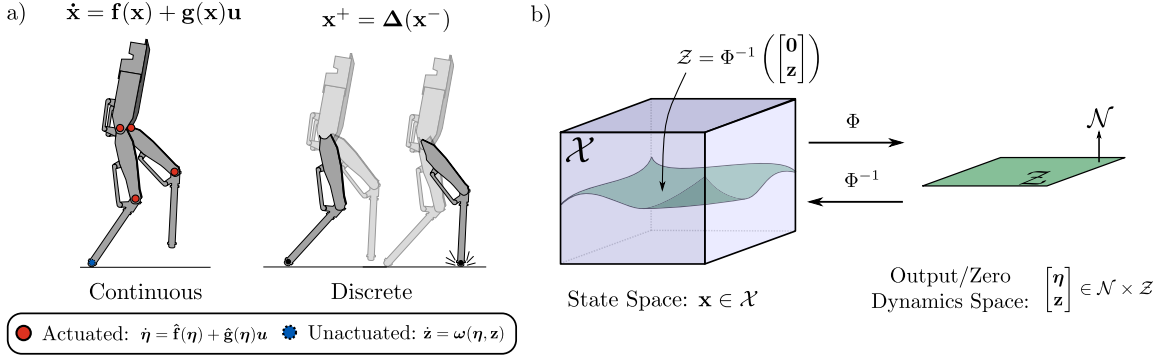}
    \vspace{-20pt}
    \caption{
    \textbf{a)} The continuous and discrete phases of the robot, with actuated ($\bm \eta$) and unactuated ($\bm z$) coordinates depicted. 
    \textbf{b)} The diffeomorphism $\bm \Phi$, and the relationship between state space coordinates and output/zero dynamics coordinates.
    }
    \label{fig:hzd_decomposition}
\end{figure}
We can then construct a diffeomorphism\footnote{. This is differentiable in the first argument and differentiable almost everywhere in the second argument.} $\bm \Phi:\pX\times\boldsymbol\Theta\to\mathcal{N}\times \mathcal{Z}$:
\begin{align*}
\begin{bmatrix}\bm \eta \\ \bm z\end{bmatrix} = \begin{bmatrix}\bm \Phi_{\bm \eta}(\bm x; \bm \theta) \\ \bm \Phi_{\bm z}(\bm x)\end{bmatrix} \triangleq \bm \Phi(\bm x; \boldsymbol \theta),~~\qquad\bm x = \bm \Phi^{-1}\left(\begin{bmatrix}\bm \eta\\ \bm z\end{bmatrix}; \boldsymbol \theta\right),
\end{align*}
as shown in Figure~\ref{fig:hzd_decomposition}b. Under this coordinate transformation, the system dynamics become:
\begin{align}
\label{eq:diff_dyn_projection}
    \begin{bmatrix}
    \dot{\bm \eta} \\ \dot{\bm z}
    \end{bmatrix} = 
    \begin{bmatrix}
    \bm{\hat{f}}(\bm \eta; \bm \theta) + \bm{\hat{g}}(\bm \eta; \bm \theta)\bm u \\
    \bm{\omega}(\bm \eta, \bm z; \bm \theta)
    \end{bmatrix},
\end{align}
where $\bm{\hat{f}}, \bm{\hat{g}}$ and $\bm{\omega}$ are the projection of the dynamics through the diffeomorphism $\bm \Phi$.%

The \textit{zero dynamics manifold} $\mathcal{Z}\subset\pX$ is thus the space where errors have been driven to zero:
\begin{align*}
\mathcal{Z} = \{\bm x\in \gX:~\bm \eta (\bm x) = 0\},
\end{align*}
as seen in Figure~\ref{fig:hzd_decomposition}a. Observe that for $\bm z\in\mathcal{Z}$, we have that $\dot{\bm z} = \bm{\omega}(\bm 0, \bm z;\bm \theta)$. The power of the method of zero dynamics lies in that it allows for guarantees about the full nonlinear dynamics by considering only a subspace of significantly smaller dimensionality \citep{isidori_elementary_1995}. Notice that although the input does not appear in the zero dynamics $\bm \omega$ in \eqref{eq:diff_dyn_projection}, the parameters of the policy $\bm \theta$ do. This realization motivates the use of the policy as a way to influence the zero dynamics and enforce the desired barrier functions, as introduced below. Finally, in this work we will learn residual dynamics $\varepsilon(\bm z)$ on the zero dynamics manifold for a corrected zero dynamics $\dot{\bm z} = \bm{\omega}(\bm 0, \bm z) + \varepsilon(\bm z)$ that compensate for modeling error. This process is captured in the episodic iteration of \cref{fig:overview}.

\subsection{Hybrid Dynamics}
Walking consists of continuous evolution with discrete impacts occurring as contact is made and broken (e.g, the feet with the ground). This sequence of continuous and discrete dynamics is shown in Figure~\ref{fig:ZDandBarriers}a can be modeled in the language of \textit{hybrid systems} as:
\begin{align*}
    \hybrid = \begin{cases}\dot{\bm x} = \bm f(\bm x) + \bm g(\bm x) \bm u & \bm x\in \mathcal{D}\backslash\mathcal{S} \\ \bm x^+ = \Delta(\bm x^-) & \bm x\in\mathcal{S}\subset\mathcal{D},\end{cases}
\end{align*}
where $\mathcal{D} \subset\mathcal{X}$ is the domain where $\gs(t)$ evolves. The \textit{guard}, $\mathcal{S}\subset\mathcal{X}$, corresponds to the set of states where the foot comes in contact with the floor. The \textit{reset map}, $\Delta:\mathcal{S} \to \mathcal{D}$ models the instantaneous sign flip of velocities observed when two rigid bodies collide (the foot with the ground). 
Furthermore, $\hybrid$ can be projected through the diffeomorphism $\bm \Phi$ to exploit the  decomposition into output and zero dynamics. For more details, we refer to \citep{westervelt_feedback_2018}.

\begin{figure}[!t]
    \centering
    \includegraphics[width=\textwidth]{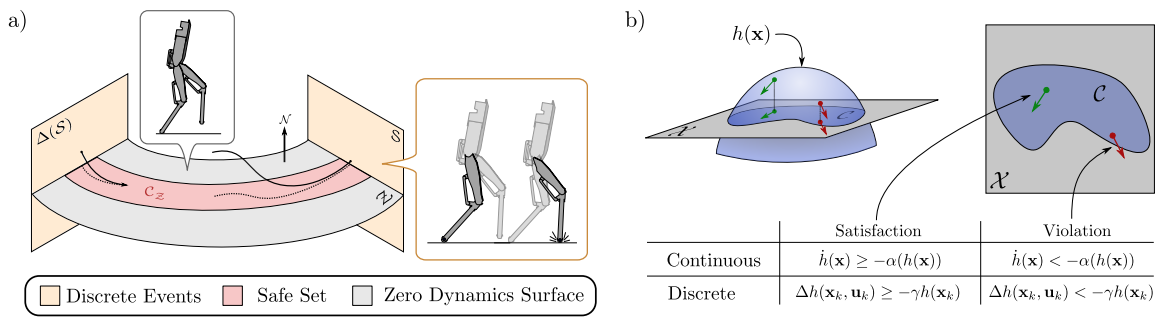}
    \vspace{-20pt}
    \caption{\textbf{a)} The guard $\mathcal{S}$, reset map $\Delta(\mathcal{S})$, and safe set $\mathcal{C}_\mathcal{Z}$ are visualized in the state space decomposed into output $\mathcal{N}$ and zero dynamics $\mathcal{Z}$ coordinates. \textbf{b)} A safe set defined as regions of the state space (gray square) where $h$ is positive (blue region). Satisfying the CBF condition implies that the flows/discrete updates of the system will not leave the safe set (although may approach the boundary). Violations imply a flow that could potentially leave the safe set.}
    \label{fig:ZDandBarriers}
\end{figure}

\subsection{Control Barrier Function Certificates}
\label{sec:cbf}
Barrier function certificates allow us to make the notion of \textit{safety} rigorous in the context of the dynamical system in \cref{eq:general_dyn}. 
We begin by specifying a set that we wish to render safe:
\begin{align}
   \safeSet = \{\gs \in \gX : h(\gs) \ge 0\} \subset \gX,
   \label{eq:safe_set}
\end{align}
where $h:\gX \to \R$ is a continuously differentiable function. In the case of bipedal walking, safe sets can describe conditions such as admissible torso angles and reasonable foot placement as shown in the first block of \cref{fig:overview}. We assume that the compact set $\safeSet$ is nonempty, has a non-empty interior,
and does not contain isolated fixed-points. We say that $\safeSet$ is \textit{safe} or \textit{forward invariant} if $x(t_0) \in \safeSet$ implies that $x(t) \in \safeSet$ for all $t \ge t_0$. With this, we have the following condition for safety (see \citep{ames_control_2019} for a brief history of this approach):
\begin{definition}[\textit{Control Barrier Function (CBF), \citep{ames2016control}}]
Consider $\safeSet$ as defined in \cref{eq:safe_set} where the continuously differentiable function $h$ has nonvanishing gradients $Dh(\gs) \neq 0$ for all $\bm \gs$ in the boundary of $\safeSet$ defined as $\partial \safeSet = \{\gs \in \gX : h(\gs) = 0\}$. The function $h$ is a \textit{Control Barrier Function} (CBF) for \cref{eq:general_dyn} on $\safeSet$ if there exists $\alpha\in\mathcal{K}_{\infty,e}$ such that for all $\gs\in \safeSet$:
\begin{align}
\label{eqn:CBF}
\dot{h}(\gs) = \frac{\partial h}{\partial \gs}(\gs)\gf(\gs,\gu) \ge -\alpha(h(\gs)).
\end{align}
\end{definition}

A function $\alpha$ is in the family of class-$\mathcal{K}_{\infty, e}$ functions if for all $a<b$,  $\alpha(a) < \alpha(b)$, $\alpha(0) = 0$, $\lim_{a\to\infty} \alpha(a) = \infty$ and $\lim_{a\to-\infty} \alpha(a) = -\infty$. In defining the CBF, we can parameterize the set of all feedback controllers guaranteeing safety as:
\begin{align}
    K_{cbf}(\gs) = \left\{\gu \in \gU : \dot h(\gs, \gu) \ge -\alpha(h(\gs))\right\}.
    \label{eq:k_cbf_set}
\end{align}
Similarly, this notion can be extended to discrete-time dynamical systems via:
\begin{align}
\label{eqn:CBF_disc}
    \Delta h(\bm x_k, \bm u_k)  \triangleq h(\bm x_{k+1}) - h(\bm x_k) \ge -\gamma h(\bm x_k),~~0<\gamma\le 1,
\end{align}
as seen in Figure~\ref{fig:ZDandBarriers}b. This leads to the following necessary and sufficient condition for safety:
\begin{theorem}[\textit{Control Barrier Function Certificates, \citep{ames2016control}}]
Given a feedback controller $\gu=k(\gs)$, the set $\safeSet$ is safe if and only if $\gu(\gs) \in K_{cbf}(\gs)$.
\end{theorem}

\section{Neural Gaits: Locomotion as a Barrier Satisfiability Problem}
\label{sec:proposed_method}

We now present our Neural Gaits approach, as depicted in \cref{fig:overview}. Instead of taking a controller-design perspective, we will take one of reference trajectory design -- specifically, we fix a controller structure $\bm u(\bm x;\bm \theta)$, which is parameterized by $\bm \theta$ through the definition of $\bm y(\bm x; \bm \theta)$.
Our method relies on the assumption that good walking can be characterized as a forward invariant set.
Thus, the first step of the method requires us to define a set of barrier functions that imply good walking.
In the following discussion, we will only consider barrier functions defined on the zero dynamics surface, i.e. $h:\mathcal{Z}\to\mathbb{R}$ as defined when the error coordinates are zero ($\bm \eta = \bm 0)$.
Importantly, the guarantees made on $\mathcal{Z} \subset \gX$ have relevance to the full state space, as is made precise in Section~\ref{sec:guarantees}.

After constructing a collection of barrier functions, we train a policy $\bm{y_d}$ that ensures the system stays safe by minimizing the violation of the barrier function conditions \eqref{eqn:CBF} and \eqref{eqn:CBF_disc} over regions of the state space.
The resulting policy renders the intersection of the safe set for all barriers forward invariant.
Finally, to mitigate model mismatch, we train a residual term $\varepsilon({\bm z})$ on the zero dynamics.
These corrected zero dynamics are then used to refine the existing policy episodically until the desired walking performance is achieved.

\subsection{Learning the Policy \texorpdfstring{$\bm{y_d}$}{yd}}

Our learning approach builds upon and unifies two lines of work. 
The first studies how to characterize good walking behavior as set invariance  via a collection of barrier function candidates \citep{ames_first_2017}.
The second studies how to train neural ODEs to satisfy control-theoretic properties such as Lyapunov stability \citep{lyanet}, which we extend to the barrier setting.

\paragraph{Learning in the Zero Dynamics.}
Recall from \cref{sec:prelim:output_and_zdyn} that the error and zero dynamics coordinates are computed from the states using the diffeomorphism $\bm \Phi$, which
only depends on the policy through $\bm \Phi_{\bm \eta}$.
We thus parameterize the policy as a function of the projection of the state onto the zero dynamics manifold and parameters $\bm \theta$, i.e. $\bm{y_d}(\bm x) = \bm{y_d}(\bm \Phi_{\bm z}(\bm x); \bm \theta)$. In other words, $\bm y_d$ only depends on the unactuated degrees of freedom of the system (e.g., the unactuated joint in \cref{fig:hzd_decomposition}) rather than the full state.
Therefore, when there is no error (i.e. $\bm \eta = 0$) we have that $\bm z\in\mathcal{Z}$ with dynamics $\dot{\bm z} = \bm \omega(\bm 0, \bm z; \bm \theta)$.
Note, importantly, that even when the error coordinates are zero, the zero dynamics are still a function of $\bm{y_d}$ and therefore $\bm \theta$. This implies that the zero dynamics are influenced by the parameters of the policy even though the control inputs are not present in $\bm \omega$.

\paragraph{Learning to Satisfy Barrier Conditions.}
Taking inspiration from \citep{hsu2015control} and \citep{ames_first_2017}, we assume that walking can be characterized as set invariance via a collection of barrier function candidates $\mathcal{H} = \{h_i\}_{i=1}^N$ (see \cref{table:barriers} discussed later in \cref{sec:application}).
To each $h_i$, we associate a \textit{region at risk}  $ \overline{\cal{S}}_i \subseteq \cal{Z}$ where the barrier function is enforced. We define a set of neural network parameters that render the region at risk safe under the barrier definition:
\begin{align}
\label{eq:safe_yd_params}
    \bm \Theta_i = \{\bm \theta \in \bm \Theta : \forall_{\bm{z} \in \overline{\cal{S}}_i} \quad \dot h_i(\bm z ; \bm \theta) \ge -\alpha(h_i(\bm z;\bm \theta))\}.
\end{align}
In other words, each $\bm \Theta_i$ corresponds to the set of policy parameters that render the set $\overline{\mathcal{S}}_i$ safe.

Thus, our learning problem is equivalent to finding a set of parameters $\bm \theta \in \bigcap_{i=1}^N \bm \Theta_i$ that render the system safe in all regions at risk.  Similar to the Lyapunov Loss studied in \citep{lyanet}, we introduce the concept of \textit{Barrier Loss} as a learning signal for training:
\begin{definition}[Barrier Loss]
For a set of barrier function candidates $\mathcal{H} = \{h_i\}_{i=1}^N$ and corresponding regions at risk $\overline{\cal{S}}_i \subset \mathcal{Z}$ on the zero dynamics, a Barrier Loss, $\loss: \bm \Theta \to \R_{\geq 0}$, is defined as:
\begin{align}
    \loss(\bm\theta) = \sum_{i=1}^N \int_{\overline{\cal{S}}_i} \max\{0,  - \dot h_i(\bm z ;  \bm \theta) - \alpha(h_i(\bm z;\bm \theta)) \} d\bm{z}. \label{eq:loss}
\end{align}
\vspace{-4mm}
\label{def:loss}
\end{definition}
When a choice of parameters $\bm\theta$ achieves zero Barrier Loss, then the safety of the zero dynamics is guaranteed by satisfying the forward invariance condition of the barrier functions:
\begin{theorem}[Zero Barrier Loss Implies Safety of Zero Dynamics]
\label{thm:zero_barrier_loss}
The zero dynamics is guaranteed to be safe in all its regions at risk 
 if and only if we find a $\bm{\theta}^*$ that attains  $\loss(\bm{\theta}^*) = 0$. 
\end{theorem}
\begin{proof}
Notice that for all $i \in \{1 \ldots N \}$ both $\dot{h}_i$  and $\alpha \circ h_i$ are continuous functions. This implies that for all $\bm z \in \mathcal{Z}$ and $\bm \theta \in \bm \Theta$, $\max\{0,  - \dot h_i(\bm z ;  \bm \theta) - \alpha(h_i(\bm z;\bm \theta)) \}$ is a continuous non-negative real function. It is well known that a continuous non-negative real function will have zero integral if and only if it is the zero function. We specialize this statement for the terms in our loss as follows:
\begingroup\makeatletter\def\f@size{10.2}\check@mathfonts
\begin{align}
    \forall_{\bm z \in \overline{\cal{S}}_i} \max\{0,  - \dot h_i(\bm z ;  \bm \theta) - \alpha(h_i(\bm z;\bm \theta)) \} = 0 \Leftrightarrow \int_{\overline{\cal{S}}_i} \max\{0,  - \dot h_i(\bm z ;  \bm \theta) - \alpha(h_i(\bm z;\bm \theta)) \} d\bm{z} = 0
\end{align}
\endgroup
It is clear that the sum in $\loss(\bm\theta)$ will be zero if an only if each integral term is zero since each integral is a non-negative function. Thus we can conclude that $\loss(\bm\theta^*)=0$  if and only if $$\forall_{i \in \{1 \ldots N \}, \bm z \in \overline{\cal{S}}_i} \max\{0,  - \dot h_i(\bm z ;  \bm \theta^*) - \alpha(h_i(\bm z;  \bm \theta^*)) \} = 0.$$ For any barrier $h_i$ and $\bm z \in \cal{Z}$ you can see that  $\max\{0,  - \dot h_i(\bm z ;  \bm \theta^*) - \alpha(h_i(\bm z;\bm \theta^*)) \} = 0$ implies that: 
\begin{align}
    - \dot h_i(\bm z ;  \bm \theta^*) - \alpha(h_i(\bm z;\bm \theta^*)) \leq 0 \implies
    \dot h_i(\bm z ;  \bm \theta^*) \geq -\alpha(h_i(\bm z;\bm \theta^*)), \nonumber
\end{align}
i.e. 
the safety condition for the barrier is satisfied.
\end{proof}

\subsection{Instantiation for Bipedal Walking}
\label{sec:application}
\cref{table:barriers} describes the barrier functions used in our experiments, which take inspiration from \citep{ames_first_2017}. 
We depict some of these conditions on the robot in \cref{fig:barriers_and_pssf}a.
As all barrier functions $h_i : \mathcal{X} \to \mathbb{R}$ are enforced on the zero dynamics surface, we will write them implicitly as $h_i~\circ~\bm \Phi^{-1}(\bm 0, ~\cdot~;\bm \theta) : \mathcal{Z} \to \R$ for $i \in \{1 \ldots N\}$ with $N=5$ in this instantiation. 
In \cref{table:barriers}, $t_A$ represents the torso angle, and $p_x$ and $p_z$ represent the $x$ and $z$ position of the swing foot, respectively. In addition to continuous time conditions, various conditions needed to be enforced on the guard, namely enforcing the location of the guard, symmetry of the model before and after impact, and a guard mapping condition. 
Interestingly, although these barriers would be relative degree two on the full state dynamics, they are directly enforceable as relative degree one barriers on the zero dynamics. This can be seen by treating $\bm {y_d}$ as the input to the zero dynamics, and observing that the zero dynamics themselves are functions of $\bm {y_d}$.

\begin{table} [b]
\begin{center}
\begin{small}
\begin{tabular}{|c|c|c| }
 \hline
 Torso Angle & $\{\bm z\in \mathcal{Z}_\mathcal{O}\}$ & $-\frac{\pi}{10} \le \theta_t(\bm z) \le 0.05$ \\ 
 \hline
Swing Foot Clearance & $\{\bm z\in \mathcal{Z}_\mathcal{O}\}$ & $0 \le (p_x(\bm z)-c_x)^2 + (p_z(\bm z)-c_z)^2 - r^2 \le 0.3$ \\
 \hline
 Impact Mapping & $\{\bm z\in \mathcal{S}_\epsilon\}$ & $-0.15 \le \Delta(\bm z) + \bm z \le 0.15$ \\
 \hline
 Symmetry & $\{\bm z\in \mathcal{S}_\epsilon\}$ & $\bm y(\bm z) = \bm y(\Delta(\bm z))$\\
 \hline 
 Foot on Guard & $\{\bm z\in \mathcal{S}_\epsilon\}$ & $p_z(\bm z) = 0$ \\
 \hline
\end{tabular}
\end{small}
\caption{Barrier functions used to characterize bipedal walking, and the associated regions at risk in which they are enforced. The first two are enforced over the continuous dynamics, and the bottom three in a buffered region of the guard. The strict equality on symmetry and the foot on guard conditions were also enforced as a training loss. \label{table:barriers}}
\end{center}
\end{table}

Note that these barrier functions $h$ are defined over the space $\mathcal{Z}$, as, given a policy $\bm{y_d}(~\cdot~;\bm{\theta}):\mathcal{Z}\to\mathbb{R}^4$, the mapping $\bm \Phi^{-1}:\mathcal{N}\times\mathcal{Z}\to\pX$ is uniquely defined. 
We take inspiration from reduced order models, and specifically the notion of orbital energy \citep{pratt_derivation_2007} to define a set $\mathcal{Z}_\mathcal{O} \subset \mathcal{Z}$ with reasonably bounded orbital energies as our first region at risk. 
We also define the set $\mathcal{S}_\epsilon \subset \mathcal{Z}_\mathcal{O}$ which contains the part of the guard in  $\mathcal{Z}_\mathcal{O}$ as well as a small region around it where discrete-time guard conditions are enforced. We learn policies that satisfy the barrier conditions on these regions of the zero dynamics by penalizing the violation of the constraints shown in \cref{table:barriers}. Notice that penalizing guard constraints over a region results in policies that are robust to impact modeling error since the policy must be prepared to change stance foot at any point in $\mathcal{S}_\epsilon$ rather than just the guard $\mathcal{S}$.

\textbf{Learning Optimization Details.}
Evaluating the Barrier Loss in \cref{eq:loss} requires solving an integral that is in general intractable.
Instead, we use Monte Carlo sampling to approximate the integral.
Since our approach follows Algorithm 1 of \citep{lyanet} we refer to it for more details while noting that we optimize for the Barrier Loss rather than the Lyapunov Loss.
A key ingredient in the Monte Carlo sampling approach in \citep{lyanet} is defining a compact support set to sample from (i.e., where the barrier condition should be satisfied). In our work this compact support set directly corresponds to the region at risk for each barrier condition.
\vspace{-2mm}
\subsection{Learning the Residual Zero Dynamics \texorpdfstring{$\bm \varepsilon(\bm z)$}{varepsilon(z)}}
\label{sec:learn_eps}
As outlined in \cref{fig:overview}, we can improve upon the nominal zero dynamics model by collecting trajectories of the robot executing the resulting policy in hardware.  We can then use those trajectories to learn a residual error term on the zero dynamics $\hat{\dot{\bm z}} = \omega(0, \bm z ; \bm \theta) + \varepsilon(\bm z)$ where $\varepsilon$ is the learned residual term. We model this residual term using Neural ODEs \citep{chen_neural_2019}, which are naturally compatible with our policy learning approach. We can iterate this process multiple times, alternating between learning $\bm \theta$ and $\varepsilon$ until the resulting policy achieves the desired behavior.

\begin{figure}[t!]
    \centering
    \includegraphics[width=\textwidth]{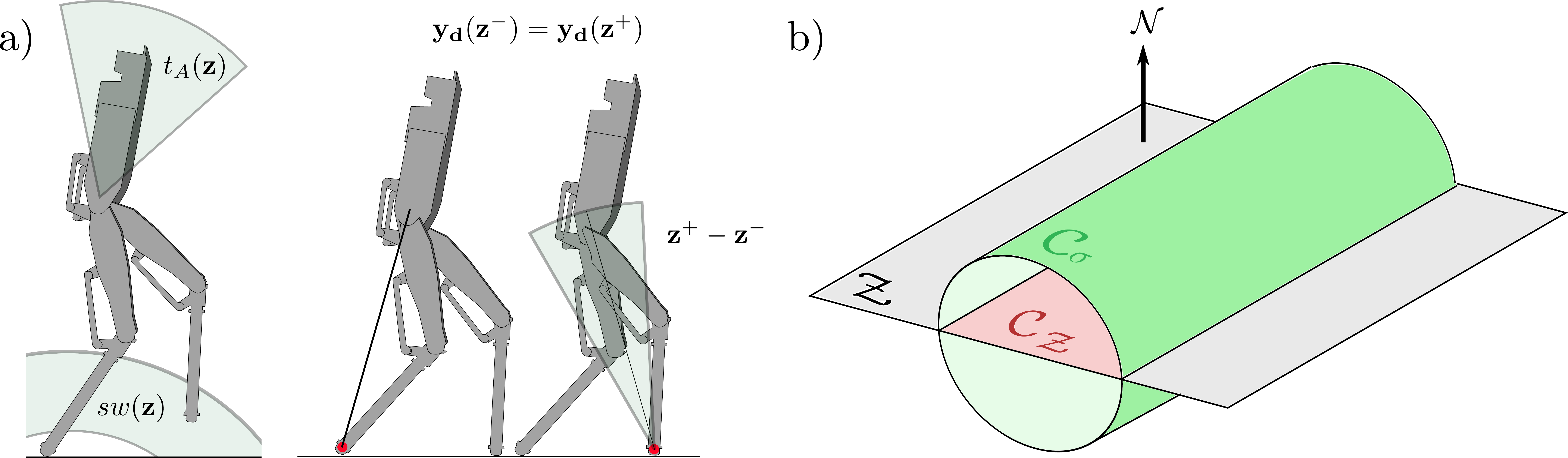}
    \vspace{-20pt}
    \caption{\textbf{a)} A depiction of the barrier functions used to enforce walking as set invariance. On the left are the two continuous time barrier conditions, and on the right the three barrier conditions enforced at the guard. The red dot on the foot indicates the stance foot, and \textbf{b)} The safe set on the zero dynamics $\mathcal{C}_{\mathcal{Z}}$, as certified by the proposed learning method, and the combined safe set $\mathcal{C}_{\sigma}$, and described in Theorem 5.}
    \label{fig:barriers_and_pssf}
\end{figure}
\vspace{-2mm}
\subsection{Providing Guarantees in the Full State Space}
\label{sec:guarantees}
Assuming a controller which exponentially converges the outputs $\bm y(\bm x)$, for example a feedback linearizing or control Lyapunov function based controller, the converse Lyapunov theorem allows us to construct a Lyapunov function $V_{\bm\eta}:\mathcal{N}\to\mathbb{R}$ verifying the exponential convergence of the outputs. Along with a certificate of safety $h_Z:\mathcal{Z} \to \mathbb{R}$ on the zero dynamics space, we can construct a set in the combined space $\mathcal{N}\times \mathcal{Z}$ which is safe, and has a barrier function certificate. This is described in the following theorem.

\begin{theorem}%
\label{thm:zero_dyn_barriers_pssf}
Let $V_{\bm \eta}=\bm \eta^\top\bm P\bm \eta:\mathcal{N}\to\mathbb{R}$ be an exponential control Lyapunov function for the output dynamics with $\dot V_{\bm \eta} \le -\gamma V_{\bm \eta}$ and $h_Z:\mathcal{Z}\to\mathbb{R}$ be a barrier function on the zero dynamics with safe set $\mathcal{C}_\mathcal{Z}$. Then, there exists a constant $\sigma\ge 0$ and $c\ge 0$ such that if $\dot h_Z(z) \ge -\alpha h_Z(z) + c$ with $\alpha\le \frac{\gamma}{2}$, the barrier function $ h (\bm \eta, \bm z) =  h_Z(\bm z) - \sigma V_{\bm \eta}(\bm \eta)$ is safe with set $\mathcal{C}_\sigma$.%

\end{theorem}
\begin{proof}
First note that the derivative of the function is given by:
\begin{align}
    \dot{ h} &= \frac{\partial  h_Z}{\partial \bm z} (\bm z) \bm w(\bm\eta , \bm z) - \sigma\dot V_{\bm \eta}(\bm \eta) \notag \\
    &\ge -\alpha  h_Z(\bm z)+c - \left|\frac{\partial  h_Z}{\partial \bm z} (\bm z) \left(\bm w(\bm\eta , \bm z)-\bm w(0 , \bm z)\right)\right| + \sigma\gamma V_{\bm \eta}(\bm \eta) \notag \\
    &\ge -\alpha  h(\bm \eta, \bm z) + c- L_{ h_Z}L_{\bm \omega_{\bm \eta}}\|\bm\eta\|_2+ \frac{\sigma\gamma}{2} \lambda_{min}(\bm P)\|\bm \eta\|_2^2, \label{eqn:RobustSafeSet}
\end{align}
where the third line follows from Cauchy Schwartz, the fact that $ h_Z$ and $\bm \omega(\bm \eta, \bm z)$ are locally Lipschitz with Lipschitz constants $L_{ h_Z}$ and $L_{\bm \omega_{\bm \eta}}$, respectively, converse Lyapunov, and the assumption that $\alpha\le \frac{\gamma}{2}$. Taking $\beta_1 = L_{ h_Z}L_{\bm \omega_{\bm \eta}}$, and $\beta_2 = \frac{\gamma}{2} \lambda_{min}(\bm P)$, we observe that $-\beta_1\|\bm\eta\|_2 + \sigma \beta_2\|\bm \eta\|_2^2 \ge -\frac{\beta_2^2}{4\sigma\beta_3}\triangleq c$. By taking $c$ defined as such, we achieve the desired result.
\end{proof}

The above theorem motivates the perspective of this work: satisfying barrier function certificates in the zero dynamics enables reasoning about safe sets in the complete state space. Note that the hybrid case is not addressed here, and is an interesting direction for future theoretical work. 

\vspace{-2 mm }\section{Simulation and Experimental Results}

\begin{figure}[t!]
    \centering
    \includegraphics[width=\textwidth]{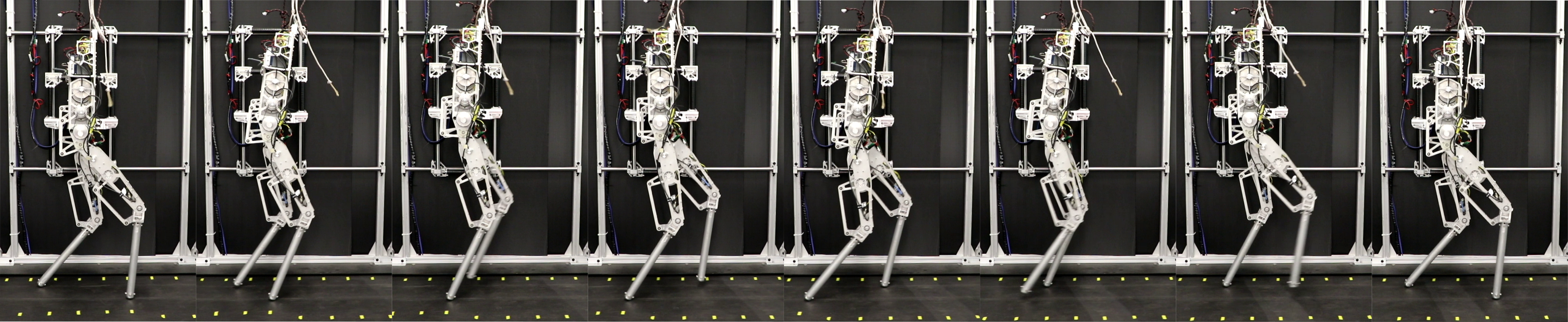}
    \vspace{-20pt}
    \caption{Gait tiles of the neural network encoding the final trained policy running in real time on the AMBER-3M robot. For a video discussing the methodology and summarizing the hardware results, please refer to \citep{video_noauthor_supplementary_nodate}.}
    \label{fig:amber_walking}
\end{figure}
The hardware platform used in this work was the planar underactuated biped AMBER-3M \citep{ambrose_toward_2017}, which has actuators on the hips and knees, and point contact feet. Both in simulation, where the RaiSim \citep{8255551} environment was used, and on hardware, the pipeline went as follows: the zero dynamics coordinate $\bm z$ was estimated, the Neural Network policy $\bm y_d(\bm z;\bm \theta)$ was evaluated, and the desired output values were passed to a PD controller running at 1kHz.
The policy $\bm y_d(\bm z; \bm \theta)$ was randomly initialized and was trained for 1000 epochs. The AdamW optimizer was used in PyTorch \citep{NEURIPS2019_9015} with an initial learning rate of $10^{-2}$, weight decay of $10^{-4}$,  with a learning rate decay schedule at epochs 100, 400, and 800. Initially, the ''gait" had the robots leg flailing randomly in the air, and when integrated resulted in the robot falling over. Once the loss converged, the policy had a loss in the order of $5\times10^{-3}$, and was able to walk stably in the simulation. The neural network ran in closed loop on the hardware platform and was called at approximately 500 Hz to produce desired outputs for the system to track. Unlike simulation, once tested on hardware, the policy resulted in the robot stumbling forward, unable to walk without falling. Data was collected over various trials, after which the methodology proposed in \cref{sec:learn_eps} was used to learn the residual of the model uncertainty, as projected to the zero dynamics space. During this process, Adam and other SGD methods were numerically unstable even with gradient clipping, so Nero \citep{liu_learning_2021} was used instead. 

Once a residual term was learned, a new policy $\bm y_d(\bm z, \bm \theta)$ was trained with the updated dynamics (warm started with the policy from the previous iteration). After convergence, the gait was again tried on hardware. The gait was significantly more stable, and able to walk without assistance; however, the gait was not robust to walking speeds. Therefore, the process was repeated, and again a new policy was learned. When testing that policy, the robot was able to walk on its own, and was robust to different walking speeds. A sample gait is shown on \cref{fig:amber_walking}. The complete code can be found here \citep{git_rodriguez_learning_2021}.

\vspace{-3mm}\section{Conclusion}

In this work, barrier functions, machine learning, and dimension reduction via zero dynamics were combined to provide a novel way of generating walking behaviors for a bipedal robot. 
Our approach used learning in two places: policy design and residual dynamics modeling via data collection on hardware.  
The proposed method culminated in a demonstration of agile and robust locomotion on hardware. Future work includes studying more complex robots, online learning, as well as policy learning for new behaviors (e.g., walking up stairs).

\clearpage
\section{Acknowledgements}
The authors would like to thank Min Dai, Ryan Cosner, and Andrew Taylor for their insightful discussions related to walking, barrier functions, and projection to state safety.

\bibliography{main}

\end{document}